\documentclass[10pt]{article}

\usepackage[preprint]{neurips_2023}
\makeatletter
\renewcommand{\@noticestring}{}
\makeatother

\usepackage[utf8]{inputenc}
\usepackage[T1]{fontenc}
\usepackage{hyperref}
\usepackage{url}
\usepackage{booktabs}
\usepackage{amsfonts}
\usepackage{amsmath}
\usepackage{amssymb}
\usepackage{nicefrac}
\usepackage{microtype}
\usepackage{xcolor}
\usepackage{graphicx}
\usepackage{algorithm}
\usepackage{algorithmic}
\usepackage{enumitem}
\usepackage{amsthm}
\usepackage{ragged2e}
\usepackage{placeins}
\usepackage{float}
\usepackage{fancyhdr}
\usepackage{tikz}
\usetikzlibrary{shapes.geometric, arrows.meta, positioning, fit, backgrounds, calc}

\newtheorem{definition}{Definition}
\newtheorem{theorem}{Theorem}
\newtheorem{lemma}{Lemma}
\newtheorem{proposition}{Proposition}

\newtheorem{assumption}{Assumption}

\newcommand{\R}{\mathbb{R}}
\newcommand{\sig}{\mathrm{sig}}

\fancypagestyle{plain}{\pagestyle{fancy}}
\fancypagestyle{empty}{%
  \fancyhf{}
  
  \fancyfoot[C]{\footnotesize Preprint --- April 2026\hfill \thepage}
}

\begin{document}


\title{
\textbf{Skyline-First Traversal as a Control Mechanism for Multi-Criteria Graph Search}\\[0.3em]
\large{A Structural Theory under Constrained Cost Models}
}

\author{
Nicolas Tacheny \\
Ni2 Innovation Lab, Li\`ege, Belgium \\
University of Mons (UMONS), Belgium \\
\texttt{nicolas.tacheny@\{ni2.com,\,umons.ac.be\}}
}

\maketitle


\begin{abstract}
In multi-criteria graph traversal, paths are compared via Pareto dominance---an ordering that identifies which paths are non-dominated, but says nothing about which path to expand next or when the search may stop.
As a result, existing approaches rely on external mechanisms---heuristics, scalarization, or population-based exploration---while Pareto dominance remains confined to passive roles such as pruning or ranking.

This paper shows that, under constrained cost models---finite cost grids, Markovian transitions, and a nonzero progress measure---Pareto geometry alone is sufficient to drive both scheduling and termination.

We show that extracting exclusively from the first Pareto layer---the skyline---induces a deterministic descent in a discrete completion potential, ensuring monotone progress toward solution completion.
In parallel, a vector lower-bound certificate provides a stopping condition that guarantees dominance coverage of all remaining traversals, without requiring a predefined number of solutions.

Our analysis establishes deterministic potential descent, certified termination via dominance coverage, a uniform bound on layer width induced by cost-grid geometry, and greedy cost-space dispersion within the skyline.
The resulting framework operates without scalarization, heuristic guidance, or probabilistic models, and repositions Pareto dominance from a passive filter to a deterministic driver of search.
\end{abstract}


\section{Introduction}

Multi-criteria traversal---in the sense of~\cite{tacheny2025parametric}---requires evaluating candidate solutions under several competing objectives simultaneously.
Each path $p$ carries a cost vector $C(p) \in \mathbb{R}_{\ge 0}^d$, and componentwise dominance induces only a partial order on the set of candidates.
In the scalar case, a priority queue resolves both scheduling and termination.
In the multi-criteria case, these two questions become structurally open: \emph{which path should be expanded next?}\ and \emph{when can exploration safely stop?}

The parametric traversal framework of~\cite{tacheny2025parametric} maintains a frontier of paths and leaves the extraction policy as a configurable parameter.
This paper designs such a policy.
Rather than proposing a new traversal algorithm, this work identifies
when the Pareto structure of the frontier can serve as a complete
scheduling mechanism — under explicit structural assumptions introduced
in Section~\ref{sec:model}.
Our guarantees are not universal but structural: they hold under explicit assumptions on the cost space (bounded discrete grid), the accumulation model (Markovian transitions), and the feasibility dynamics (nonzero progress measure).
Within this scope, the dominance geometry of the frontier serves not merely as a filter, but as a \emph{deterministic scheduling mechanism} that simultaneously governs extraction order, cost-space coverage, and certified termination, without breaking the convergence dynamic toward a feasible solution.
We address the problem of computing a Pareto-representative set of feasible traversals without relying on scalarization or heuristic guidance.

\paragraph{Structured cost spaces.}
In real-world infrastructure systems---such as telecom or datacenter networks---path search is not performed over arbitrary graphs.
Costs are discretized by design, transitions are constrained by local engineering rules, and feasible paths exhibit structured progress toward completion.
These properties induce a highly regular cost space that generic multi-objective formulations typically do not exploit.
This paper formalizes such structure and shows that it is sufficient to turn Pareto geometry into a structural control mechanism for traversal.
The key insight is that the frontier element closest to completion always lies in the skyline --- so restricting extraction to the first Pareto layer preserves convergence.

\paragraph{Approach.}
The frontier $\mathcal{F}(t)$ decomposes into successive Pareto layers.
We show that extracting exclusively from the first layer (the skyline) yields two core guarantees.
The central result is a \emph{deterministic potential descent} (Theorem~\ref{thm:deterministic-layer-first}): the frontier element closest to a feasible solution always lies in the skyline, so a skyline-only policy never loses access to it; a natural integer potential measuring residual completion distance is monotonically non-increasing, with strict decrease each time this minimum-potential element is extracted.
Second, a \emph{certified stopping criterion} (Theorem~\ref{thm:certified-stop}): a vector lower-bound certificate detects when no future path can improve the current solution set beyond the quantization granularity, so the algorithm terminates without a predefined number of desired solutions.
Together, these two mechanisms establish a \emph{dominance coverage} property: the returned solution set dominates every feasible path up to quantization granularity.
Our guarantees are collective rather than incremental: individual solutions discovered during the search are not required to be globally non-dominated at discovery time; optimality is established at termination through the stopping certificate.
No scalarization, total ordering, or probabilistic model is required for either result.
Unlike multi-objective A*-like approaches, which rely on heuristic guidance or scalarization, the results show that under structural conditions, no additional guidance is required beyond dominance: we replace incremental optimality guarantees with a certified global completeness result.

\paragraph{Contributions.}
The paper makes the following contributions:
\begin{enumerate}[nosep]
\item \textbf{Deterministic potential descent} (Theorem~\ref{thm:deterministic-layer-first}). The frontier element closest to a feasible solution always lies in the skyline. A natural integer potential is monotonically non-increasing under any skyline-only policy, with strict decrease at each extraction of this element. Skyline restriction preserves convergence.
\item \textbf{Certified stopping criterion} (Theorem~\ref{thm:certified-stop}). A vector lower-bound certificate detects when no future path can improve the current solution set. Together with the descent, this establishes a \emph{dominance coverage} property: the returned set dominates every feasible path.
\item \textbf{Uniform layer width bound} (Proposition~\ref{thm:layer-width}). Each Pareto layer contains a bounded number of elements per signature, determined by the cost-grid cardinality. This keeps the instantaneous decision space tractable and yields a runtime governed by cost-grid geometry rather than path combinatorics.
\item \textbf{Greedy cost-space coverage}. The bin-coverage tie-breaking rule ensures that the policy never revisits an already-explored region of cost space when an unexplored one is available in the current skyline, promoting geometric dispersion of explored solutions.
\end{enumerate}

\paragraph{Reading guide.}
Section~\ref{sec:model} introduces the level-lifted Markovian cost model and the hypotheses it requires.
Sections~\ref{sec:dominance}--\ref{sec:pareto-geometry} develop the dominance geometry of the frontier: Pareto layers, cost-faithful quantization, and the width bound that controls the decision space at each step.
Section~\ref{sec:layered-frontier-scheduling} defines the Skyline-First extraction policy built on this geometry.
Section~\ref{sec:coverage-optimality} analyzes the geometric coverage behavior induced by the bin-coverage tie-breaking rule.
The two core guarantees follow: certified stopping (Section~\ref{sec:certified-stopping}) and deterministic potential descent (Section~\ref{sec:deterministic-guarantees}), establishing that skyline restriction certifies termination and preserves convergence.
A running example is developed incrementally throughout the paper.

\newpage

\section{Related Work}
\label{sec:related_work}

Pareto dominance is central in multi-objective search~\cite{salzman2023heuristic}, but its role is consistently restricted to passive functions: output construction, local pruning, or static selection.
In none of these settings does dominance serve as an explicit, global control mechanism for scheduling and termination.

In classical multi-objective shortest path (MOSP) algorithms~\cite{martins1984multicriteria} and their heuristic-guided extensions MOA*~\cite{stewart1991multiobjective} and NAMOA*~\cite{mandow2010multiobjective}, Pareto structure serves as an \emph{output}: these methods enumerate non-dominated labels at each node, enforcing incremental optimality --- each label, once settled, is guaranteed Pareto-optimal among all paths reaching that node.
This design fundamentally decouples Pareto structure from search control: dominance does not influence which regions of the search space are explored first.
The same algorithms also use dominance as a \emph{local filter}, pruning paths that are componentwise dominated by existing ones~\cite{mandow2010multiobjective}.
This reduces the search space, but the pruning operates locally at each node: it does not affect the global exploration order or search dynamics.
We separate these two roles --- output construction and local filtering --- because they have distinct impacts on search behavior, even when combined within a single algorithm such as NAMOA*.

A third line of work bypasses dominance entirely.
Scalarization methods --- weighted sums, $\varepsilon$-constraints, and other aggregation schemes~\cite{ehrgott2005multicriteria} --- reduce the multi-objective problem to a sequence of scalar optimizations, discarding the dominance geometry entirely and exploring only a subset of the Pareto front determined by the chosen weights or constraints.
By collapsing the cost space into a scalar objective, these methods lose the structural information that the present work seeks to exploit.

Related approaches such as resource-constrained shortest paths and bounded-cost multi-objective search also exploit the discrete, bounded nature of the cost space, typically through explicit budgets or admissible bounds.
In these settings, dominance and cost thresholds are used to prune infeasible or suboptimal paths, but Pareto structure remains a filtering mechanism: it discards dominated candidates without governing the exploration order or determining when to stop.

The skyline operator~\cite{borzsonyi2001skyline}, introduced in the database community, selects all non-dominated tuples from a relation --- the same geometric operation that defines our first Pareto layer.
However, the skyline operator identifies non-dominated points but does not prescribe how they should be explored: it addresses \emph{selection} over a static dataset, not \emph{exploration} of a dynamically expanding frontier.

Beyond graph search, Pareto dominance is also widely used in other multi-objective paradigms, including evolutionary optimization~\cite{deb2002nsga2} and reinforcement learning.
In these settings, dominance plays a central role in ranking, selection, or approximation of solution sets.
However, its function remains confined to evaluation mechanisms: it assesses or filters candidate solutions, but does not determine how the exploration process unfolds over time, nor which regions of the search space are prioritized.

Across all these approaches, Pareto structure remains passive --- decoupled from search control, discarded by reduction, or limited to static selection --- but never elevated to a control law governing the dynamics of the search.
This distinction is not algorithmic but structural: it concerns the role assigned to dominance within the search process.
This paper occupies the missing role.
We treat the layered Pareto geometry of the frontier as a \emph{scheduling mechanism} that simultaneously governs extraction order, convergence, and certified termination --- transforming Pareto dominance from a descriptive property into an operational principle.
A key enabler is the structural regularity of the cost space: existing formulations typically consider general cost models and do not exploit the bounded, discrete geometry that infrastructure-constrained systems naturally exhibit.

\newpage
\section{Level-Lifted Traversal Model}
\label{sec:model}

We begin by fixing the base model and the structural cost assumption on which all subsequent results depend.
We formulate the model using paths and edges for readability; all constructions extend directly to general traversals over arbitrary transitions (see Section~\ref{sec:traversal-instance}).

\subsection{Graph and Cost Model}

Let $G = (N,E)$ be a finite directed graph with source node $s \in N$ and a set of target nodes $T \subseteq N$.
The set of \emph{paths} on~$G$ is
\[
\mathcal{P}(G) \;=\; \bigl\{\,(e_1, \dots, e_k) \;\big|\; k \ge 0,\; \forall\, i < k,\; \mathrm{dst}(e_i) = \mathrm{src}(e_{i+1})\bigr\},
\]
where $\mathrm{src}(e)$ and $\mathrm{dst}(e)$ denote the source and destination nodes of an edge~$e$, and $k=0$ gives the empty path starting at any node.
For a non-empty path $p = (e_1, \dots, e_k)$, we write $\mathrm{start}(p) = \mathrm{src}(e_1)$ and $\mathrm{end}(p) = \mathrm{dst}(e_k)$ for its initial and terminal nodes.
Given two non-empty paths $p, r \in \mathcal{P}(G)$ with $\mathrm{end}(p) = \mathrm{start}(r)$, their \emph{composition} $p \circ r \in \mathcal{P}(G)$ is the concatenation of the two edge sequences.

The main idea behind our cost model is that each path from $s$ to $T$ carries a $d$-dimensional cost vector that accumulates non-negatively along discrete levels, subject to a componentwise budget.
The discretization is natural: even inherently continuous costs (e.g.\ monetary expenditure) are in practice billed or budgeted in brackets, so a finite grid faithfully captures the operational granularity.
Each path also carries a compact state — the \emph{signature} — that summarizes everything the future cost dynamics need to know about the path history.
Because dominance comparisons are only meaningful between paths that can be extended in the same way, the signature must at least identify the current node.
It may encode additional context---e.g.\ visited resources or current operating mode---but always refines the terminal node.
When the path is extended by one edge, both the cost vector and the signature update according to a Markov rule: only the current state and the chosen edge determine the next state.
More precisely:
\begin{enumerate}[label=(\roman*),nosep]
\item Each path accumulates a $d$-dimensional cost vector $C(p) \in \R_{\ge 0}^d$.
\item The cost evolution depends on path context, summarized by a signature $\sig(p)$.
\item Both update by a Markov rule: only the current state $(\sig(p), C(p))$ and the chosen edge $e$ determine the next state.
\end{enumerate}

We begin by introducing the core objects manipulated during the search:
paths, their signatures, and their accumulated cost vectors.
\begin{definition}[Level-lifted Markovian cost model]
\label{def:lifted-markovian}
A \emph{level-lifted Markovian cost model} on $G$ is a $d$-dimensional cost function $C : \mathcal{P}(G) \to \mathbb{R}_{\ge 0}^d$
\[
C(\emptyset) = \mathbf{0}, \qquad C_i(p \circ e) = \delta_i(\sig(p),\, e,\, C_i(p)) \quad \text{for each } i \in \{1,\dots,d\},
\]
specified by the following components:
\begin{enumerate}[label=(\alph*),nosep]
\item A finite state space $\Sigma$ and a signature map 
\[
\sig : \mathcal{P}(G) \to \Sigma
\]
satisfying $\sig(p \circ e) = \alpha(\sig(p), e)$ for a transition function $\alpha : \Sigma \times E \to \Sigma$. In particular, $\sig$ refines the terminal node: $\sig(p) = \sig(q) \Rightarrow \mathrm{end}(p) = \mathrm{end}(q)$.
\item A componentwise budget $B \in \mathbb{R}_{\ge 0}^d$: a path $p$ respects the budget if $C_i(p) \le B_i$ for every $i \in \{1,\dots,d\}$.
\item For each dimension~$i$, a finite subset $\mathcal{G}_i \subset [0, B_i]$ containing $0$ and $B_i$, called the \emph{cost grid} of dimension~$i$. The full cost grid is $\mathcal{G} := \mathcal{G}_1 \times \cdots \times \mathcal{G}_d$.
\item A transition function $\delta_i : \Sigma \times E \times \mathcal{G}_i \to \mathcal{G}_i$ for each dimension~$i$, satisfying $\delta_i(\sigma, e, g) \ge g$ and monotone in~$g$: $g \le g' \Rightarrow \delta_i(\sigma, e, g) \le \delta_i(\sigma, e, g')$.
\end{enumerate}
\end{definition}

Consider the dynamics of a path that grows one edge at a time: at each step the pair $(\sig(p), C(p))$ is a Markov state, since extending $p$ by an edge $e$ yields $\sig(p \circ e) = \alpha(\sig(p), e)$ and $C_i(p \circ e) = \delta_i(\sig(p), e, C_i(p))$, both determined by $(\sig(p), C(p))$ and $e$ alone.
The model is \emph{lifted} because the original graph $G$ is embedded into the augmented state space $\Sigma$, which absorbs path-dependent information.
It is \emph{level}-based because each cost dimension $i$ tracks a discrete level $g \in \mathcal{G}_i$: the transition function $\delta_i$ maps grid levels to grid levels, so cost accumulation operates on a finite lattice rather than on~$\mathbb{R}_{\ge 0}$.
The classical additive model is recovered by setting $\delta_i(\sigma, e, g) = g + c_i(e)$ for a fixed edge-cost function $c : E \to \mathbb{R}_{\ge 0}^d$, with $\mathcal{G}_i$ chosen so that all partial sums remain on the grid. 

\paragraph{Running example.}
To illustrate how the structural assumptions of the model arise in practice, we introduce a small infrastructure network (Figure~\ref{fig:running-graph}) that will serve as a running example throughout the paper.
Consider a graph with four nodes $\{s, a, b, t\}$ and three cost criteria ($d=3$):
\begin{enumerate}[label=($c_{\arabic*}$),nosep]
\item \emph{complexity} — additive with fixed edge weights;
\item \emph{length} — additive with fixed edge weights;
\item \emph{zone-switching cost} — each time the zone of the current edge differs from that of the previous edge, a unit penalty is incurred.
\end{enumerate}
Each edge belongs to a \emph{zone} drawn from $\{\mathsf{Z_1},\mathsf{Z_2}\}$ — for instance, in a fiber-optic network, two geographic regions between which switching incurs a reconfiguration cost.
This makes $c_3$ depend on the zone of the last traversed edge, so the signature must record it: $\sig(p) = (\mathrm{end}(p),\, \mathrm{zone}(e_k))$ for $p = (e_1, \dots, e_k)$, with $\sig(p_s) = (s, \bot)$ for the trivial path, where $\bot$ denotes the initial state before any edge has been traversed.
The per-dimension cost grids are $\mathcal{G}_1 = \{0,1,2\}$, $\mathcal{G}_2 = \{0,1,2,3,4\}$, $\mathcal{G}_3 = \{0,1\}$, and the budget is $B = (2,4,1)$.
This example is representative of a broader class of infrastructure-constrained graph models in which costs are quantized and transitions depend on local state.

\begin{figure}[ht]
\centering
\begin{tikzpicture}[scale=0.85, every node/.style={font=\small},
  vertex/.style={circle, draw, minimum size=8mm, inner sep=0pt, font=\normalsize},
  zone label/.style={font=\scriptsize\bfseries},
  cost label/.style={font=\scriptsize, fill=white, inner sep=1pt}
]
  \node[vertex] (s) at (0,0) {$s$};
  \node[vertex] (a) at (3,1.6) {$a$};
  \node[vertex] (b) at (3,-1.6) {$b$};
  \node[vertex] (t) at (6,0) {$t$};

  \draw[-{Stealth}] (s) -- node[zone label, above left, pos=0.2] {$\mathsf{Z_1}$}
    node[cost label, below right, pos=0.5] {$\scriptstyle +1,\!+2\mathrm{km}$} (a);
  \draw[-{Stealth}] (a) -- node[zone label, above right, pos=0.8] {$\mathsf{Z_1}$}
    node[cost label, below left, pos=0.5] {$\scriptstyle +1,\!+2\mathrm{km}$} (t);
  \draw[-{Stealth}] (a) -- node[zone label, left, pos=0.5] {$\mathsf{Z_2}$}
    node[cost label, right, pos=0.5] {$\scriptstyle +1,\!+1\mathrm{km}$} (b);
  \draw[-{Stealth}] (s) -- node[zone label, below left, pos=0.2] {$\mathsf{Z_2}$}
    node[cost label, above right, pos=0.5] {$\scriptstyle +1,\!+1\mathrm{km}$} (b);
  \draw[-{Stealth}] (b) -- node[zone label, below right, pos=0.8] {$\mathsf{Z_2}$}
    node[cost label, above left, pos=0.5] {$\scriptstyle +0,\!+1\mathrm{km}$} (t);
\end{tikzpicture}
\caption{Running example: edge zones and per-edge cost increments in complexity ($c_1$) and length ($c_2$). The zone-switching cost $c_3$ is accumulated via the signature.}
\label{fig:running-graph}
\end{figure}
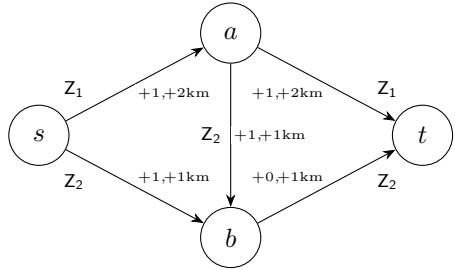

\subsection{Traversal Instance}
\label{sec:traversal-instance}

The cost model above describes per-edge dynamics.
We now package it with a graph into a single object that defines the search problem.

\begin{definition}[Level-lifted Markovian traversal instance]
\label{def:traversal-instance}
A \emph{level-lifted Markovian traversal instance} consists of a finite directed graph $G = (N,E)$,
a level-lifted Markovian cost model on $G$ (Definition~\ref{def:lifted-markovian}),
a source node $s \in N$, and a target set $T \subseteq N$.
\end{definition}

We fix a level-lifted Markovian traversal instance $(G, C, s, T)$.
The search objective is to find paths from~$s$ to~$T$ that respect the budget~$B$; we formalize this as follows.

\begin{definition}[Feasibility]
\label{def:feasibility}
Given a traversal instance $(G, C, s, T)$, we define three related notions.
\begin{enumerate}[label=(\alph*),nosep]
\item A path $p \in \mathcal{P}(G)$ is \emph{feasible} if
\[
\mathrm{start}(p) = s, \qquad \mathrm{end}(p) \in T, \qquad \forall\, i \in \{1,\dots,d\},\; C_i(p) \le B_i.
\]
\item A \emph{feasible extension} of a path $p \in \mathcal{P}(G)$ is a path $r \in \mathcal{P}(G)$ such that
\[
p \circ r \in \mathcal{P}(G) \qquad \text{and} \qquad p \circ r \text{ is feasible.}
\]
\item A path $p \in \mathcal{P}(G)$ is \emph{completable} if there exists a feasible extension of~$p$.
\end{enumerate}
\end{definition}

We assume that at least one feasible path exists.
This assumption is necessary: the potential-descent guarantee and certified stopping criterion developed in this paper are only meaningful when a feasible path exists.

We apply the parametric traversal algorithm of~\cite{tacheny2025parametric}, which maintains a \emph{frontier} $\mathcal{F}(t)$ — the set of paths generated but not yet expanded at time $t$ — and iteratively selects a frontier element, expands it by one transition, and inserts the resulting successors back into the frontier.
Successors that violate the budget are pruned and never inserted, so all frontier elements satisfy $C_i(p) \le B_i$ for every $i \in \{1,\dots,d\}$.
The extraction policy is a parameter of the algorithm; this paper develops one grounded in Pareto layer structure.

More generally, the framework of~\cite{tacheny2025parametric} defines a \emph{transition} as any pair $(n,m)$ of nodes---either an edge transition ($(n,m) \in E$) or a gap transition ($(n,m) \notin E$)---and a \emph{traversal} as a finite sequence of such transitions.
For readability, we use the terminology of paths and edges in the remainder of the paper, corresponding to the special case of edge-connected graphs; all results extend directly to the general setting.

Since the cost model is Markovian, two paths that share the same signature and cost vector are indistinguishable from the perspective of future dynamics: any extension available to one is equally available to the other, with identical cost evolution.
Maintaining both would therefore be redundant, which motivates the following assumption.

\begin{assumption}[Frontier representation invariant]
\label{ass:frontier-invariant}
For each signature in $\Sigma$, the frontier maintains at most one representative per cost vector in $\mathcal{G}$.
If two paths $p, q \in \mathcal{P}(G)$ satisfy $\sig(p) = \sig(q)$ and $C(p) = C(q)$, only one representative is kept.
This is enforced by a pruning rule at insertion time.
\end{assumption}

All subsequent definitions, lemmas, and theorems are stated within this setting unless noted otherwise.

\subsection{Structural Progress Assumption}

The cost model allows transitions with zero increment in all dimensions, so the search could explore arbitrarily many edges without measurable cost-space progress.
In infrastructure networks, however, each transition typically consumes some irreducible resource --- physical distance, an operation slot, or a unit of capacity --- making zero-cost wandering impossible in practice.
We formalize this observation as a structural progress condition.

\begin{assumption}[Nonzero progress measure]
\label{ass:min-increment}
There exists a non-empty set of \emph{progressive dimensions} $I_{\min} \subseteq \{1,\dots,d\}$ and strictly positive constants $\delta_{\min,i} > 0$ for each $i \in I_{\min}$ such that
\[
\delta_i(\sigma,e,g) - g \;\ge\; \delta_{\min,i} \quad \text{for all } (\sigma,e,g).
\]
In words: at least one cost dimension enforces a strictly positive increment at every step.
The system is \emph{dissipative} in the sense that progress resource is irreversibly consumed at every transition, making infinite wandering impossible.
\end{assumption}

More broadly, the assumptions of the cost model are consistent with structured graph models arising in infrastructure systems, where costs are quantized and transitions are governed by local engineering constraints.

Two quantities derived from Assumption~\ref{ass:min-increment} play a central role in the analysis.

\begin{definition}[Minimum increment vector and step bound]
\label{def:min-increment-vector}
The \emph{minimum increment vector} $\delta_{\min} \in \mathbb{R}_{\ge 0}^d$ is defined by
\[
(\delta_{\min})_i \;:=\;
\begin{cases}
\delta_{\min,i} & \text{if } i \in I_{\min}, \\
0 & \text{otherwise.}
\end{cases}
\]
Every edge increases the cost by at least $\delta_{\min,i}$ in each progressive dimension $i \in I_{\min}$.
The \emph{maximum step count} is
\[
\Lambda \;:=\; \min_{i \in I_{\min}} \left\lfloor \frac{B_i}{\delta_{\min,i}} \right\rfloor.
\]
Since $B_i$ is finite and each step consumes at least $\delta_{\min,i}$, no path can exceed $\Lambda$ steps.
In particular, the frontier $\mathcal{F}(t)$ remains finite at every step.
\end{definition}

\paragraph{Running example (continued).}
The running example satisfies Assumption~\ref{ass:min-increment} with $I_{\min} = \{2\}$ and $\delta_{\min,2} = 1\,\mathrm{km}$: every edge in the network increases the length $c_2$ by at least $1\,\mathrm{km}$, so no path can extend indefinitely.
With budget $B_2 = 4$, the maximum number of steps is $\Lambda = \lfloor 4/1 \rfloor = 4$.

\newpage

\section{Dominance Geometry}
\label{sec:dominance}

We now develop the dominance and layering structure of the frontier.
All definitions in this section operate within the traversal instance fixed in Section~\ref{sec:model}.

\subsection{Per-Signature Frontier and Dominance}
\label{sec:dominance-layers}

Because paths with different signatures face different future cost dynamics, dominance comparisons are only meaningful among paths that share the same signature.
We therefore partition the frontier by signature before defining dominance.

\begin{definition}[Per-signature frontier]
\label{def:per-signature-frontier}
For each signature $\sigma \in \Sigma$, the \emph{per-signature frontier} is
\[
\mathcal{F}_\sigma(t) := \{p \in \mathcal{F}(t) : \sig(p) = \sigma\}.
\]
The global frontier decomposes as $\mathcal{F}(t) = \biguplus_{\sigma} \mathcal{F}_\sigma(t)$.
Let $\Sigma_\mathcal{F}(t) := \{\sigma \in \Sigma : \mathcal{F}_\sigma(t) \neq \emptyset\}$ denote the set of active signatures at time $t$.
\end{definition}

Within each per-signature frontier, we compare paths by componentwise cost.

\begin{definition}[Dominance]
\label{def:dominance}
For $p,q \in \mathcal{F}_\sigma(t)$, $p$ \emph{weakly dominates} $q$, written $p \preceq q$, if
\[
C(p)_i \le C(q)_i \quad \forall\, i.
\]
If in addition strict inequality holds for at least one component, $p$ \emph{strictly dominates} $q$, written $p \prec q$.
\end{definition}

\subsection{Pareto Layers}

Dominance induces a natural stratification of each per-signature frontier into successive non-dominated shells.
Removing the first non-dominated set reveals a second, and so on, yielding a sequence of \emph{Pareto layers} that will serve as the basis for scheduling.

\begin{definition}[Pareto layers]
\label{def:pareto-layers}
Within each per-signature frontier $\mathcal{F}_\sigma(t)$, the successive \emph{Pareto layers} are defined by iterative non-dominated extraction.
The first layer --- the \emph{skyline} --- is the set of non-dominated elements of the frontier:
\[
\mathcal{L}_{1,\sigma}(t) := \{ p \in \mathcal{F}_\sigma(t) \mid \nexists q \in \mathcal{F}_\sigma(t),\; q \prec p \}.
\]
For $k \ge 1$, the \emph{residual set} at depth $k$ and the subsequent layer are defined by
\[
\mathcal{R}_{k+1,\sigma}(t) := \mathcal{R}_{k,\sigma}(t) \setminus \mathcal{L}_{k,\sigma}(t), \qquad
\mathcal{L}_{k+1,\sigma}(t) := \{ p \in \mathcal{R}_{k+1,\sigma}(t) \mid \nexists q \in \mathcal{R}_{k+1,\sigma}(t),\; q \prec p \},
\]
where $\mathcal{R}_{1,\sigma}(t) := \mathcal{F}_\sigma(t)$.
The global $k$-th Pareto layer is the union over all active signatures:
\[
\mathcal{L}_k(t) := \bigcup_{\sigma \in \Sigma_\mathcal{F}(t)} \mathcal{L}_{k,\sigma}(t).
\]
This yields the disjoint decomposition $\mathcal{F}(t) = \biguplus_{k \ge 1} \mathcal{L}_k(t)$.
\end{definition}

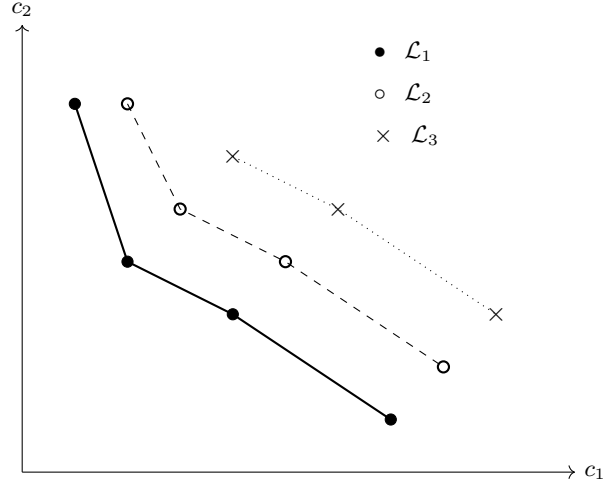
\begin{figure}[h!]
\centering
\begin{tikzpicture}[scale=0.7, every node/.style={font=\small}]
  \draw[->] (0,0) -- (10.5,0) node[right] {$c_1$};
  \draw[->] (0,0) -- (0,8.5) node[above] {$c_2$};

  \draw[thick] (1,7) -- (2,4) -- (4,3) -- (7,1);
  \filldraw (1,7) circle (3pt);
  \filldraw (2,4) circle (3pt);
  \filldraw (4,3) circle (3pt);
  \filldraw (7,1) circle (3pt);

  \draw[dashed] (2,7) -- (3,5) -- (5,4) -- (8,2);
  \draw[thick] (2,7) circle (3pt);
  \draw[thick] (3,5) circle (3pt);
  \draw[thick] (5,4) circle (3pt);
  \draw[thick] (8,2) circle (3pt);

  \draw[dotted] (4,6) -- (6,5) -- (9,3);
  \node[font=\normalsize] at (4,6) {$\times$};
  \node[font=\normalsize] at (6,5) {$\times$};
  \node[font=\normalsize] at (9,3) {$\times$};

  \node[right] at (6.5,8) {$\bullet$ \; $\mathcal{L}_1$};
  \node[right] at (6.5,7.2) {$\circ$ \; $\mathcal{L}_2$};
  \node[right] at (6.5,6.4) {$\times$ \; $\mathcal{L}_3$};
\end{tikzpicture}
\caption{Schematic: Pareto layer decomposition in two-dimensional cost space (single signature).
Filled discs form the first layer $\mathcal{L}_1$ (skyline); removing them reveals the second layer $\mathcal{L}_2$ (open circles), and so on.
Lines connect consecutive points within each layer, revealing the successive shells.}
\label{fig:pareto-layers}
\end{figure}

\paragraph{Running example (continued).}
Returning to the network of Figure~\ref{fig:running-graph}, the three $s$-to-$t$ paths, their accumulated costs, and signatures are:

\medskip
\begin{center}
\begin{tabular}{@{}lcccc@{}}
\toprule
\textbf{Path} & \textbf{Complexity} ($c_1$) & \textbf{Length} ($c_2$) & \textbf{Zone Switch} ($c_3$) & \textbf{Signature} \\
\midrule
$p_1 = s \xrightarrow{\mathsf{Z_1}} a \xrightarrow{\mathsf{Z_1}} t$ & 2 & 4 & 0 & $(t,\mathsf{Z_1})$ \\
$p_2 = s \xrightarrow{\mathsf{Z_2}} b \xrightarrow{\mathsf{Z_2}} t$ & 1 & 2 & 0 & $(t,\mathsf{Z_2})$ \\
$p_3 = s \xrightarrow{\mathsf{Z_1}} a \xrightarrow{\mathsf{Z_2}} b \xrightarrow{\mathsf{Z_2}} t$ & 1 & 4 & 1 & $(t,\mathsf{Z_2})$ \\
\bottomrule
\end{tabular}
\end{center}
\medskip

\noindent Path $p_3$ incurs a zone-switching cost of $1$ at the transition from edge $s \to a$ (zone~$\mathsf{Z_1}$) to edge $a \to b$ (zone~$\mathsf{Z_2}$).
Because the signature includes the last edge zone, $p_1$ has signature $(t,\mathsf{Z_1})$ while $p_2$ and $p_3$ share signature $(t,\mathsf{Z_2})$.
Dominance comparisons are only meaningful within the same signature (Definition~\ref{def:per-signature-frontier}).
Within signature $(t,\mathsf{Z_2})$: $p_2 = (1,2,0)$ dominates $p_3 = (1,4,1)$ componentwise, so $p_3$ is relegated to the second layer.
Hence $\mathcal{L}_{1,(t,\mathsf{Z_1})} = \{p_1\}$ and $\mathcal{L}_{1,(t,\mathsf{Z_2})} = \{p_2\}$, giving $|\mathcal{L}_1| = 2$.
This illustrates per-signature layering: although $p_2$ dominates $p_3$ within signature $(t,\mathsf{Z_2})$, the structurally distinct path $p_1$ is preserved because it belongs to a different signature $(t,\mathsf{Z_1})$.
This decomposition highlights the structural role of the first Pareto layer: it gathers all minimal trade-offs that cannot be improved simultaneously across all cost dimensions.

\subsection{Cost Quantization}
\label{sec:cost-quantization}

To use the skyline as a scheduling mechanism, we need to control how many elements it can contain at any given time.
Each Pareto layer is finite---the cost grid is finite and no two layer elements with the same signature share the same cost vector---but the number of incomparable vectors can still grow with the grid size.
By partitioning the cost space into resolution bins, we can count the maximum number of structurally distinct positions a layer can occupy, which will yield an explicit bound on layer width.

Recall that the budget $B = (B_1, \dots, B_d)$ bounds each cost dimension, and each cost grid $\mathcal{G}_i$ lives in $[0, B_i]$.
A quantization maps each interval $[0, B_i]$ to a finite set of integer levels.

\begin{definition}[Cost quantization]
\label{def:quantization}
A \emph{cost quantization} is a vector $\varphi = (\varphi_1,\dots,\varphi_d)$ where each $\varphi_i : [0, B_i] \to \{0,1,\dots,m_i\}$ is non-decreasing with $\varphi_i(0) = 0$ and $\varphi_i(B_i) = m_i$.
The integer $m_i$ represents the number of resolution intervals in dimension~$i$.
It induces the following concepts:
\begin{enumerate}[label=(\alph*),nosep]
\item The \emph{cell} at level $j$ in dimension~$i$ is $\mathcal{C}_{i,j} := \varphi_i^{-1}(j) = \{x \in [0,B_i] : \varphi_i(x) = j\}$.
\item The \emph{bin index vector} of a path $p$ is
$\beta(p) := \bigl(\varphi_1(C(p)_1),\;\dots,\;\varphi_d(C(p)_d)\bigr) \in \{0,\dots,m_1\}\times\cdots\times\{0,\dots,m_d\}$.
\item The total number of bin index vectors is $B^* := \prod_{i=1}^d (m_i + 1)$.
\end{enumerate}
\end{definition}

\subsection{Cost-Faithful Quantization}

The quantization $\varphi$ introduced above is general.
Since the cost grid $\mathcal{G}_i$ is finite and already ordered, the quantization can always be chosen fine enough to separate all grid levels---this is a natural requirement, not a restrictive one, because collapsing two distinct grid values into the same bin would discard information that the cost model explicitly encodes.
We formalize this as our third structural assumption.

\begin{assumption}[Cost-faithful quantization]
\label{ass:dominance-monotone}
The quantization $\varphi$ separates all grid levels: for each dimension $i$, $\varphi_i$ is strictly monotone on $\mathcal{G}_i$, i.e.\ for all $x, y \in \mathcal{G}_i$,
\[
x < y \;\;\Longrightarrow\;\; \varphi_i(x) < \varphi_i(y).
\]
\end{assumption}

This assumption leaves the choice of $\varphi$ open.
The simplest quantization that satisfies it---the \emph{rank quantization}---assigns each grid value its ordinal position: since $\mathcal{G}_i$ is finite and totally ordered, this produces the finest possible binning with no information loss.
Concretely, define $\mathrm{rank}_i : \mathcal{G}_i \to \{0,\dots,|\mathcal{G}_i| - 1\}$ as the ordinal rank in $\mathcal{G}_i$, and extend to $[0, B_i]$ by
\[
\varphi_i(x) \;:=\; \max\bigl\{\, \mathrm{rank}_i(y) : y \in \mathcal{G}_i,\; y \le x \,\bigr\}.
\]
Then $\varphi_i$ is non-decreasing on $[0, B_i]$, $\varphi_i(0) = 0$, $m_i = |\mathcal{G}_i| - 1$, and
\[
B^* \;=\; \prod_{i=1}^d |\mathcal{G}_i|.
\]
The rank quantization satisfies Assumption~\ref{ass:dominance-monotone} by construction: since $\mathrm{rank}_i$ is the ordinal position in the finite ordered set $\mathcal{G}_i$, distinct grid values $x < y$ in $\mathcal{G}_i$ receive distinct ranks $\mathrm{rank}_i(x) < \mathrm{rank}_i(y)$, so $\varphi_i$ is strictly monotone on $\mathcal{G}_i$.


\section{Layered Pareto Geometry for Frontier Scheduling}
\label{sec:pareto-geometry}

The cost model, the Pareto layer decomposition, and the cost quantization introduced in the previous sections provide the raw ingredients.
We now derive three geometric properties of the layered frontier that together make it viable as a scheduling mechanism.

The first property ensures that the layer decomposition has no gaps: whenever the frontier is non-empty, the skyline is non-empty, so a skyline-only policy always has an element to extract.

\begin{lemma}[Layer contiguity]
\label{lem:layer-contiguity}
If $\mathcal{F}(t) \neq \emptyset$, then $\mathcal{L}_1(t) \neq \emptyset$.
More generally, if $\mathcal{L}_k(t) = \emptyset$ then $\mathcal{L}_j(t) = \emptyset$ for all $j \ge k$.
\end{lemma}
\begin{proof}
Since $\mathcal{F}(t)$ is finite, every non-empty subset admits a non-dominated element, so $\mathcal{L}_1(t) \neq \emptyset$.
For the second claim, note that for each signature $\sigma$, $\mathcal{L}_{k,\sigma}(t)$ is the set of non-dominated elements of $\mathcal{R}_{k,\sigma}(t)$.
If $\mathcal{L}_k(t) = \emptyset$, then $\mathcal{L}_{k,\sigma}(t) = \emptyset$ for every $\sigma$. If $\mathcal{R}_{k,\sigma}(t) \neq \emptyset$, since $\mathcal{R}_{k,\sigma}(t) \subseteq \mathcal{F}(t)$ is finite, it admits a non-dominated element contradicting $\mathcal{L}_{k,\sigma}(t) = \emptyset$.
Hence $\mathcal{R}_{k,\sigma}(t) = \emptyset$ for every $\sigma$.
By definition, $\mathcal{R}_{k+1,\sigma}(t) = \mathcal{R}_{k,\sigma}(t) \setminus \mathcal{L}_{k,\sigma}(t) = \emptyset$ for every $\sigma$, so $\mathcal{L}_{j,\sigma}(t) = \emptyset$ for all $j \ge k$ and all $\sigma$.
Taking the union over $\sigma$ gives $\mathcal{L}_j(t) = \emptyset$ for all $j \ge k$.
\end{proof}

The second property connects the cost quantization to the layer structure: within any single layer and signature, no two elements can occupy the same resolution bin.
This is a direct consequence of the frontier representation invariant (Assumption~\ref{ass:frontier-invariant}) and the faithfulness of the quantization (Assumption~\ref{ass:dominance-monotone}).

\begin{lemma}[Bin exclusivity]
\label{lem:bin-exclusivity}
For each signature $\sigma$, every per-signature Pareto layer $\mathcal{L}_{k,\sigma}$ contains at most one element per bin index vector.
That is, if $p,q \in \mathcal{L}_{k,\sigma}$ and $\beta(p)=\beta(q)$, then $p=q$.
\end{lemma}
\begin{proof}
Let $p,q \in \mathcal{L}_{k,\sigma}$ with $\beta(p) = \beta(q)$.
Then for each dimension~$i$, $\varphi_i(C(p)_i) = \varphi_i(C(q)_i)$.
Since $C(p)_i, C(q)_i \in \mathcal{G}_i$ and $\varphi_i$ is strictly monotone on $\mathcal{G}_i$ (Assumption~\ref{ass:dominance-monotone}), hence injective, we obtain $C(p)_i = C(q)_i$ for all~$i$, i.e.\ $C(p) = C(q)$.
Since $\sig(p) = \sig(q) = \sigma$, the frontier representation invariant (Assumption~\ref{ass:frontier-invariant}) gives $p = q$.
\end{proof}

Combining these two lemmas yields the central geometric result of the section: each Pareto layer has a width bounded by the cost-grid geometry, independently of how many paths the search has generated.
This is what makes the skyline a tractable decision space.

\begin{proposition}[Uniform layer width bound]
\label{thm:layer-width}
Define the \emph{skyline width} $W(t) := |\Sigma_\mathcal{F}(t)|\cdot B^*$.
\begin{enumerate}
\item \emph{Per-signature bound.} Every per-signature Pareto layer satisfies
\[
|\mathcal{L}_{k,\sigma}(t)|
\;\le\;
B^*
\quad \text{for every } k\ge 1,\; \sigma \in \Sigma_\mathcal{F}(t).
\]
\item \emph{Global layer bound.} The global Pareto layer satisfies
\[
|\mathcal{L}_k(t)|
\;\le\;
W(t)
\quad \text{for every } k\ge 1.
\]
\end{enumerate}
This bound holds after every expansion step, regardless of how many new frontier elements the expansion generates.
\end{proposition}
\begin{proof}
By Lemma~\ref{lem:bin-exclusivity}, each per-signature layer $\mathcal{L}_{k,\sigma}(t)$ contains at most one element per bin index vector, so $|\mathcal{L}_{k,\sigma}(t)| \le B^*$.
Since $\mathcal{L}_k(t) = \biguplus_{\sigma} \mathcal{L}_{k,\sigma}(t)$ is a disjoint union over at most $|\Sigma_\mathcal{F}(t)|$ signatures, we obtain $|\mathcal{L}_k(t)| \le W(t)$.
\end{proof}

Writing $W := \max_t W(t)$ for the peak skyline width, this bound is the structural bottleneck of the entire framework: all subsequent scheduling, coverage, and termination guarantees operate within a decision space of size at most $W$ per layer. Moreover, under the rank quantization (Section~\ref{sec:cost-quantization}), $m_i = |\mathcal{G}_i| - 1$ and
\[
B^* \;=\; \prod_{i=1}^d |\mathcal{G}_i|.
\]
So each Pareto layer can therefore contain at most as many elements (per signature) as there are cells in the cost grid.
Crucially, $B^*$ is a design parameter: choosing a coarser quantization (fewer levels $m_i$ per dimension) reduces $B^*$ and hence the per-layer width, at the cost of lower resolution in the stopping certificate (Section~\ref{sec:certified-stopping}).

\paragraph{Running example (continued).}
In the running example (Figure~\ref{fig:running-graph}), the cost grids are $\mathcal{G}_1 = \{0,1,2\}$ ($|\mathcal{G}_1| = 3$), $\mathcal{G}_2 = \{0,1,2,3,4\}$ ($|\mathcal{G}_2| = 5$), and $\mathcal{G}_3 = \{0,1\}$ ($|\mathcal{G}_3| = 2$).
Under the rank quantization, $B^* = 3 \times 5 \times 2 = 30$.
At the terminal frontier, two signatures are active: $(t,\mathsf{Z_1})$ and $(t,\mathsf{Z_2})$, so $|\Sigma_\mathcal{F}| = 2$ and $W = 2 \times 30 = 60$.
In practice the actual skyline is far smaller ($|\mathcal{L}_1| = 2$).
The gap between $W$ and the observed skyline size reflects the purely structural nature of the bound; tighter bounds incorporating graph-specific constraints are left for future work.

We now define the extraction policy that leverages this layered structure.

\section{Skyline-First Frontier Scheduling}
\label{sec:layered-frontier-scheduling}

The previous sections established that the frontier decomposes into Pareto layers of bounded width, that the skyline is always non-empty when the frontier is, and that cost quantization provides a finite grid of resolution bins.
These geometric properties make it possible to define an extraction policy that draws exclusively from the first Pareto layer and distributes its extractions across cost space---without requiring any external guidance.
We now define this policy and show in the subsequent sections that it inherits strong structural guarantees from the layered geometry: deterministic progress toward feasible solutions, certified termination, and cost-space coverage.

\subsection{Extraction Policy}

\begin{definition}[Skyline-First coverage policy]
\label{def:skyline-first}
Let $\hat{\beta}(t)$ denote the set of resolution bin indices already extracted by time $t$.
The \emph{Skyline-First policy} selects at each step an element
\[
\pi_{\mathrm{SF}}\bigl(\mathcal{F}(t),\,\hat{\beta}(t)\bigr)
\;\in\;
\operatorname*{arg\,min}_{p \,\in\, \mathcal{L}_1(t)}
\mathbf{1}\!\bigl(\beta(p) \in \hat{\beta}(t)\bigr).
\]
where $\mathbf{1}(\cdot)$ denotes the indicator function.
Among all elements of the current skyline, extract one whose resolution bin has not yet been extracted; if all bins are already represented, extract any skyline element.
Remaining ties are broken arbitrarily; none of the subsequent results depend on the tie-breaking choice.
\end{definition}

By Lemma~\ref{lem:layer-contiguity}, $\mathcal{L}_1(t)$ is non-empty whenever $\mathcal{F}(t)$ is non-empty, so $\pi_{\mathrm{SF}}$ is well-defined at every non-terminal step.

The policy combines two mechanisms that play orthogonal roles.
Skyline restriction (extracting only from $\mathcal{L}_1(t)$) drives the deterministic frontier-potential descent (Section~\ref{sec:deterministic-guarantees}); the power comes from the layered structure, not the policy itself.
Bin-coverage tie-breaking promotes geometric dispersion across cost space (Section~\ref{sec:coverage-optimality}).

\paragraph{Running example (continued).}
We trace the first steps of Skyline-First on the network of Figure~\ref{fig:running-graph}, starting with $\mathcal{F}(0) = \{(s,\bot,(0,0,0))\}$ and $\hat{\beta} = \emptyset$.

\emph{Step~1.}
The skyline is $\mathcal{L}_1 = \{(s,\bot,(0,0,0))\}$.
Its bin $\beta = (0,0,0)$ is uncovered, so the policy selects it.
Expansion produces two successors: $(a,\mathsf{Z_1},(1,2,0))$ and $(b,\mathsf{Z_2},(1,1,0))$.
Now $\hat{\beta} = \{(0,0,0)\}$.

\emph{Step~2.}
The skyline is $\mathcal{L}_1 = \{(a,\mathsf{Z_1},(1,2,0)),\; (b,\mathsf{Z_2},(1,1,0))\}$ — neither dominates the other since they have different signatures.
Both bins are uncovered; suppose the policy picks $(a,\mathsf{Z_1},(1,2,0))$.
Expansion yields $(t,\mathsf{Z_1},(2,4,0))$ — a complete feasible path, recorded as the first solution — and $(b,\mathsf{Z_2},(1,3,1))$ via the zone-switching edge $a \to b$.
Now $\hat{\beta} = \{(0,0,0),\; (1,2,0)\}$.

\emph{Step~3.}
The skyline now contains $(b,\mathsf{Z_2},(1,1,0))$ and $(b,\mathsf{Z_2},(1,3,1))$.
Within signature $(b,\mathsf{Z_2})$, $(1,1,0)$ dominates $(1,3,1)$, so only $(b,\mathsf{Z_2},(1,1,0))$ remains in the skyline.
Its bin is uncovered, so the policy extracts it and expands it toward~$t$.

At each step, the policy extracts from the skyline and prefers bins not yet visited, promoting dispersion across cost space as discussed in Section~\ref{sec:coverage-optimality}.

\subsection{Algorithm}

Under Assumption~\ref{ass:frontier-invariant}, each pair $(\sigma, c) \in \Sigma \times \mathcal{G}$ has at most one path representative in the frontier.
The frontier $\mathcal{F}$ can therefore be identified with a set of pairs $(\sigma, c)$, together with a representative map $\mathrm{rep} : \mathcal{F} \to \mathcal{P}(G)$ that returns the unique path for each pair.
The skyline $\mathcal{L}_1 \subseteq \mathcal{F}$ is the subset of non-dominated pairs --- where dominance only compares pairs with the same signature.
The search proceeds by two operations: $\mathrm{Insert}$ (Algorithm~\ref{alg:insert}) adds a newly generated path to the frontier, and $\mathrm{Extract}$ (Algorithm~\ref{alg:extract}) selects a frontier element to explore next.
Both maintain the invariant that $\mathcal{L}_1$ is the Pareto front of $\mathcal{F}$.

\begin{algorithm}[h!]
\caption{$\mathrm{Insert}(p, \mathcal{F}, \mathcal{L}_1, \mathrm{rep})$}\label{alg:insert}
\begin{algorithmic}[1]
\REQUIRE New path $p$ with $(\sig(p), C(p)) \notin \mathcal{F}$; frontier $\mathcal{F}$; skyline $\mathcal{L}_1$; representative map $\mathrm{rep}$
\STATE $\sigma^* \leftarrow \sig(p)$, \quad $c^* \leftarrow C(p)$
\STATE $\mathcal{F} \leftarrow \mathcal{F} \cup \{(\sigma^*, c^*)\}$, \quad $\mathrm{rep}(\sigma^*, c^*) \leftarrow p$
\STATE $\mathcal{L}_1 \leftarrow \{ (\sigma, c) \in \mathcal{F} \mid \nexists\, (\sigma, c') \in \mathcal{F},\; c' \prec c \}$ \hfill\COMMENT{only $\sigma^*$ affected}
\end{algorithmic}
\end{algorithm}

\begin{algorithm}[h!]
\caption{$\mathrm{Extract}(\mathcal{F}, \mathcal{L}_1, \mathrm{rep}, \hat{\beta})$}\label{alg:extract}
\begin{algorithmic}[1]
\REQUIRE Frontier $\mathcal{F}$; skyline $\mathcal{L}_1$; representative map $\mathrm{rep}$; covered bin set $\hat{\beta}$
\ENSURE Selected path $p^*$; updated $\hat{\beta}$
\IF{$\exists\, (\sigma, c) \in \mathcal{L}_1$ with $\beta(\mathrm{rep}(\sigma, c)) \notin \hat{\beta}$}
  \STATE Pick any such $(\sigma^*, c^*)$
\ELSE
  \STATE Pick any $(\sigma^*, c^*) \in \mathcal{L}_1$
\ENDIF
\STATE $p^* \leftarrow \mathrm{rep}(\sigma^*, c^*)$
\STATE $\mathcal{F} \leftarrow \mathcal{F} \setminus \{(\sigma^*, c^*)\}$
\STATE $\mathcal{L}_1 \leftarrow \{ (\sigma, c) \in \mathcal{F} \mid \nexists\, (\sigma, c') \in \mathcal{F},\; c' \prec c \}$ \hfill\COMMENT{only $\sigma^*$ affected}
\STATE $\hat{\beta} \leftarrow \hat{\beta} \cup \{\beta(p^*)\}$
\RETURN $p^*$, $\hat{\beta}$
\end{algorithmic}
\end{algorithm}

Algorithms~\ref{alg:insert} and~\ref{alg:extract} are a conceptual description of the search mechanics, not an implementation specification.
In particular, although $\mathcal{L}_1$ is defined globally, its update can be performed incrementally and only affects pairs with the modified signature $\sigma^*$, since dominance does not compare across signatures.

\section{Geometric Coverage and Skyline Dispersion}
\label{sec:coverage-optimality}

The Skyline-First policy prioritizes unexplored regions of the cost space through its bin-coverage tie-breaking rule.
While the core guarantees of the framework---deterministic descent and certified termination---do not rely on this mechanism, it plays an important role in shaping the geometric distribution of explored solutions.

\paragraph{Greedy coverage property.}
At each step, the policy selects an element from the current skyline $\mathcal{L}_1(t)$.
Among these candidates, it prioritizes elements whose resolution bin has not yet been extracted: if there exists $p \in \mathcal{L}_1(t)$ such that $\beta(p) \notin \hat{\beta}(t)$, then the policy selects such an element.
This induces a simple but important property: the policy never selects an element from an already-explored bin when an unexplored bin is available in the current skyline.
This follows directly from the definition of the selection rule (Definition~\ref{def:skyline-first}) and does not depend on any structural assumptions beyond those required to define the skyline.

\paragraph{Dispersion in cost space.}
The greedy coverage property induces a dispersion effect over the cost space.
Each extraction step expands the set of visited resolution bins whenever possible, ensuring that exploration is distributed across distinct regions of the cost grid.
Geometrically, the skyline represents a set of mutually non-dominated trade-offs.
By preferring previously unvisited bins within this set, the policy promotes a broad sampling of these trade-offs rather than repeatedly refining a narrow region.
This behavior contrasts with policies that rely on scalarization or lexicographic ordering, which may repeatedly select candidates concentrated in a specific region of the cost space.

\paragraph{Scope and limitations.}
The coverage property is inherently local: it applies to the current skyline snapshot and does not account for the evolution of the frontier over time.
In a dynamic run, each extraction modifies the frontier and may introduce new skyline elements with previously unseen bin indices.
As a result, the global coverage trajectory depends on the interaction between the skyline structure at each step, the generation of successors, and the distribution of reachable cost vectors.
Characterizing this interaction remains an open problem.

\paragraph{Role within the framework.}
The bin-coverage rule is not required for the correctness of the search.
The structural guarantees of the framework---convergence toward feasible solutions and certified termination---are entirely induced by the skyline restriction and the dominance structure of the frontier.
The coverage mechanism therefore plays a secondary but meaningful role: it influences how the search explores the cost space, promoting geometric diversity among explored solutions without affecting the underlying guarantees.
This separation highlights a key design principle: Pareto geometry governs correctness, while tie-breaking rules shape exploration behavior.

\section{Certified Stopping Criterion via Vector Lower Bounds}
\label{sec:certified-stopping}

A key question for any search process is: when can we stop?
In multi-criteria path search, the answer is not a fixed number of solutions but a structural condition on the frontier.
The intuition is the following: since costs accumulate monotonically and progressive dimensions enforce a minimum increment at every step (Assumption~\ref{ass:min-increment}), every path in the frontier carries a \emph{minimum future cost}---a vector lower bound on the cost of any feasible path that extends it.
If, for every skyline element, some already-discovered solution is at least as good as this lower bound in every dimension, then no future exploration can improve the solution set.
This yields not merely a stopping rule but a \emph{certificate}: a verifiable condition that, once satisfied, guarantees that the returned solutions dominate every feasible path that remains to be discovered.

\subsection{Minimum-Cost Lower Bound}

\begin{definition}[Discovered feasible paths]
\label{def:discovered}
Let $\mathcal{S}(t) \subseteq \mathcal{P}(G)$ denote the set of feasible paths discovered by the search up to time $t$.
\end{definition}

Any frontier element $p \in \mathcal{F}(t)$ requires at least one additional edge to reach $T$.
By Assumption~\ref{ass:min-increment}, each such edge increases the cost by at least $\delta_{\min}$ componentwise.
Therefore, for any $p \in \mathcal{F}(t)$ and any feasible extension $r$ of $p$:
\[
C(p \circ r) \;\ge\; C(p) + \delta_{\min} \quad \text{componentwise.}
\]
The quantity $C(p) + \delta_{\min}$ is thus a lower bound on the cost of $p \circ r$ for any feasible extension $r$ of~$p$.

\subsection{Certified Stopping Condition}

We can now formalize the stopping criterion outlined in the introduction to this section: the search may terminate when every skyline element is already covered by an existing solution.

\begin{definition}[Stopping certificate]
\label{def:stopping-certificate}
The search admits a \emph{resolution certificate} at time $t$ if for every skyline element $p \in \mathcal{L}_1(t)$, there exists a feasible path $s \in \mathcal{S}(t)$ such that
\[
C(s) \;\le\; C(p) + \delta_{\min} \quad \text{componentwise.}
\]
\end{definition}

Before proving the soundness of the certificate, we establish a persistence property: if a feasible path $p \notin \mathcal{S}(t)$ exists, then the frontier always contains a prefix from which it can still be reached.

\begin{lemma}[Frontier persistence of feasible paths]
\label{lem:frontier-persistence}
Let $p = e_1 \circ \cdots \circ e_m$ be a feasible path with $p \notin \mathcal{S}(t)$, with prefixes $p_k = e_1 \circ \cdots \circ e_k$.
Then the frontier contains an element that matches the signature and cost of some strict prefix of~$p$:
\[
\exists\, q \in \mathcal{F}(t),\; \exists\, k < m \;:\; (\sig(q), C(q)) = (\sig(p_k), C(p_k)).
\]
\end{lemma}
\begin{proof}
We prove the invariant by induction on the extraction steps.

\emph{Base case.}\enspace
Before any extraction, the frontier is initialized with the trivial path at $s$, so $\mathcal{F}(0)$ contains an element $q$ with $(\sig(q), C(q)) = (\sig(p_0), C(p_0))$ and the invariant holds with $k = 0$.

\emph{Inductive step.}\enspace
Suppose the invariant holds after $t$ extractions, with $q \in \mathcal{F}(t)$ satisfying $(\sig(q), C(q)) = (\sig(p_k), C(p_k))$.
After extraction $t+1$, two cases arise:
\begin{enumerate}[label=(\roman*),nosep]
\item If $q$ is not the element extracted at step $t+1$, it remains in $\mathcal{F}(t+1)$ and the invariant persists with the same $k$.
\item If $q$ is the element extracted at step $t+1$, the algorithm generates a successor $q' = q \circ e_{k+1}$.
By the Markovian property (Definition~\ref{def:lifted-markovian}), the signature and cost of $q'$ are determined by those of $q$ and the edge $e_{k+1}$:
\[
\sig(q') = \alpha(\sig(q), e_{k+1}) = \alpha(\sig(p_k), e_{k+1}) = \sig(p_{k+1}),
\]
\[
C_i(q') = \delta_i(\sig(q), e_{k+1}, C_i(q)) = \delta_i(\sig(p_k), e_{k+1}, C_i(p_k)) = C_i(p_{k+1}).
\]
Feasibility of $p$ gives $C(q') = C(p_{k+1}) \le C(p) \le B$, so $q'$ satisfies the budget and is not discarded.
If $k + 1 = m$, then $q'$ reaches a node in $T$ (since $p$ does and they share the same signature), so the algorithm records $q'$ as a complete feasible path with $C(q') = C(p)$---meaning $p$ has been discovered, contradicting the premise.
If $k + 1 < m$, then $q'$ is inserted into the frontier.
By Assumption~\ref{ass:frontier-invariant}, if an element with the same $(\sigma, C)$ is already present, one representative is kept.
In either case, $\mathcal{F}(t+1)$ contains an element $q''$ with $(\sig(q''), C(q'')) = (\sig(q'), C(q')) = (\sig(p_{k+1}), C(p_{k+1}))$, and the invariant holds with $q''$ and $k + 1$.
\end{enumerate}
In both cases, the invariant is maintained as long as $p$ remains undiscovered.
\end{proof}

We now prove that the stopping certificate is sound: once it holds, no further exploration can produce a solution that dominates any element of the current set.

\begin{theorem}[Certified resolution-completeness]
\label{thm:certified-stop}
If the stopping certificate holds at time $t^*$, then for any feasible path $s \notin \mathcal{S}(t^*)$, there exists $s^* \in \mathcal{S}(t^*)$ such that $s^* \preceq s$.
\end{theorem}

\begin{proof}
Consider any feasible path $s = e_1 \circ \cdots \circ e_m$ with $s \notin \mathcal{S}(t^*)$, with prefixes $s_k = e_1 \circ \cdots \circ e_k$.
By Lemma~\ref{lem:frontier-persistence}, $\mathcal{F}(t^*)$ contains an element $p$ with $(\sig(p), C(p)) = (\sig(s_k), C(s_k))$ for some $k < m$.
By the Markovian property (Definition~\ref{def:lifted-markovian}), the suffix $r := e_{k+1} \circ \cdots \circ e_m$ is a feasible extension of $p$, and $C(p \circ r) = C(s_k \circ r) = C(s)$ since $(\sig(p), C(p)) = (\sig(s_k), C(s_k))$.
By the minimum-cost lower bound:
\[
C(s) \;=\; C(p \circ r) \;\ge\; C(p) + \delta_{\min} \quad \text{componentwise.}
\]

We now show that there exists a skyline element $q \in \mathcal{L}_{1,\sig(p)}(t^*)$ with $q \preceq p$.
If $p \in \mathcal{L}_{1,\sig(p)}(t^*)$, set $q := p$.
Otherwise, by definition of the Pareto layer decomposition, there exists $p' \in \mathcal{F}_{\sig(p)}(t^*)$ with $p' \prec p$.
Since $\mathcal{F}_{\sig(p)}(t^*)$ is finite, iterating this argument terminates at some $q \in \mathcal{L}_{1,\sig(p)}(t^*)$ with $q \preceq p$.

Since $C(q) \le C(p)$ componentwise:
\[
C(q) + \delta_{\min} \;\le\; C(p) + \delta_{\min} \;\le\; C(s).
\]

Since $q \in \mathcal{L}_{1,\sig(p)}(t^*) \subseteq \mathcal{L}_1(t^*)$ and the stopping certificate holds at $t^*$, there exists $s^* \in \mathcal{S}(t^*)$ with $C(s^*) \le C(q) + \delta_{\min}$ componentwise.
Therefore:
\[
C(s^*) \;\le\; C(q) + \delta_{\min} \;\le\; C(s) \quad \text{componentwise,}
\]
which gives $s^* \preceq s$.
\end{proof}
\paragraph{Running example (continued).}
Returning to the network of Figure~\ref{fig:running-graph}, suppose the search has discovered solution $p_2$ with $C(p_2) = (1,2,0)$.
With $\delta_{\min} = (0,1,0)$, the certificate checks whether $C(p_2) \le C(p) + (0,1,0)$ for every skyline element $p$.
If the only remaining skyline element has cost $(1,2,0)$, then $C(p_2) = (1,2,0) \le (1,3,0) = C(p) + \delta_{\min}$, and the certificate holds: no undiscovered path can improve upon $p_2$.

This guarantee holds at the certified stopping time and does not imply incremental Pareto optimality during the search.
The resolution of the certificate is determined by the minimum increment vector $\delta_{\min}$ (Assumption~\ref{ass:min-increment}), while quantization (Assumption~\ref{ass:dominance-monotone}) is only used to control the structural width of the Pareto layers.
Moreover, the certificate controls \emph{when} to stop but does not address \emph{how quickly} the search reaches a feasible solution.
The next section introduces a potential-theoretic framework that bounds the number of effective progress steps.

\section{Deterministic Frontier-Potential Descent}
\label{sec:deterministic-guarantees}

Section~\ref{sec:certified-stopping} established that a certificate based on $\delta_{\min}$ can detect when no further exploration is needed.
A natural concern remains: \emph{does restricting extraction to the skyline break the convergence dynamic toward a solution?}
We now answer this question negatively, showing that skyline restriction induces a deterministic descent of an integer-valued completion potential.

\begin{definition}[Completion potential]
\label{def:completion-potential}
The \emph{completion potential} of a path $p \in \mathcal{F}(t)$ is
\[
H(p) := \min \{ |r| : \text{$r$ is a feasible extension of $p$} \},
\]
where $|r|$ denotes the number of edges in $r$.
If $p$ is not completable, $H(p)=+\infty$.
\end{definition}

Note that $H(p)$ depends on the existence and length of future feasible extensions, which are not known during the search.
It is therefore not computable by the algorithm and serves purely as an analytical tool for bounding progress.

A key property of $H$ is that it respects dominance: if a path $p$ has lower cost than another path $q$ with the same signature, then $p$ has at least as much residual budget in every dimension, so every feasible extension of $q$ is also feasible for $p$.
In particular, $p$ can reach a target at least as quickly as $q$.

\begin{lemma}[Residual budget monotonicity]
\label{lem:residual-budget}
If $\sig(p) = \sig(q)$ and $C(p) \le C(q)$ componentwise, then every feasible extension of $q$ is also a feasible extension of $p$. In particular, $H(p) \le H(q)$.
\end{lemma}
\begin{proof}
Let $v = \mathrm{end}(p) = \mathrm{end}(q)$ and let
$r = e_1 \circ \cdots \circ e_m$ be a feasible extension of $q$.
We show $C_i(p \circ r) \le C_i(q \circ r)$ for each $i$ by induction on the
edges of $r$.

Write $r_k = e_1 \circ \cdots \circ e_k$ for $0 \le k \le m$, so that
$r_0$ is the empty extension and $r_m = r$.

\emph{Base case.}\enspace
$C_i(p \circ r_0) = C_i(p) \le C_i(q) = C_i(q \circ r_0)$.

\emph{Inductive step.}\enspace
Suppose $C_i(p \circ r_{k-1}) \le C_i(q \circ r_{k-1})$.
Since $\sig(p) = \sig(q)$, the transition rule
$\sig(p \circ e) = \alpha(\sig(p), e)$ gives
$\sig(p \circ r_{k-1}) = \sig(q \circ r_{k-1})$ by an immediate
induction on~$k$. Both calls to $\delta_i$ therefore share the same
first two arguments $(\sig(p \circ r_{k-1}),\, e_k)$.
Monotonicity of $\delta_i$ in its third argument
(Definition~\ref{def:lifted-markovian}) then yields
\[
C_i(p \circ r_k)
= \delta_i\!\bigl(\sig(p \circ r_{k-1}),\; e_k,\; C_i(p \circ r_{k-1})\bigr)
\le \delta_i\!\bigl(\sig(q \circ r_{k-1}),\; e_k,\; C_i(q \circ r_{k-1})\bigr)
= C_i(q \circ r_k).
\]

Taking $k = m$ gives $C(p \circ r) \le C(q \circ r)$ componentwise.
Feasibility of $q \circ r$ means $C(q \circ r) \le B$, so
$C(p \circ r) \le C(q \circ r) \le B$ and $p \circ r$ is also feasible.

For the potential inequality, if $H(q) = +\infty$ the bound $H(p) \le H(q)$ is immediate.
Otherwise, let $r^*$ be a shortest feasible extension of $q$, so that $|r^*| = H(q)$.
By the above, $r^*$ is also a feasible extension of $p$, hence $H(p) \le |r^*| = H(q)$.
\end{proof}

The completion potential $H(p)$ varies across frontier elements and over time.
To bound the number of descent steps globally, we need a worst-case bound on $H$ over all skyline elements at any point during the search.

\begin{definition}[Bounded skyline potential]
\label{def:bounded-skyline-potential}
Define
\[
\bar{H}_1 \;:=\; \sup_{t,\; p \,\in\, \mathcal{L}_1(t),\; H(p)<+\infty} H(p).
\]
\end{definition}

This bound can be tightened using the dissipative structure of the cost model.

\begin{lemma}[Budget-driven potential bound]
\label{lem:budget-potential-bound}
Under Assumption~\ref{ass:min-increment}, $\bar{H}_1 \le \Lambda$.
\end{lemma}
\begin{proof}
Let $p \in \mathcal{L}_1(t)$ with $H(p) < +\infty$.
By definition of $H(p)$, there exists a feasible extension
$r = e_1 \circ \cdots \circ e_m$ with $m = H(p)$ such that $C_i(p \circ r) \le B_i$ for all~$i$.

Fix $i \in I_{\min}$ and write $r_j = e_1 \circ \cdots \circ e_j$ for the prefix of $r$ of length $j$ (with $r_0$ the empty extension). By Assumption~\ref{ass:min-increment}, each edge contributes at least
$\delta_{\min,i}$:
\[
C_i(p \circ r_j) - C_i(p \circ r_{j-1}) \;\ge\; \delta_{\min,i}.
\]
Summing the per-edge increments gives:
\[
C_i(p \circ r) \;=\; C_i(p) + \sum_{j=1}^{m} \bigl(C_i(p \circ r_j) - C_i(p \circ r_{j-1})\bigr) \;\ge\; C_i(p) + m\,\delta_{\min,i} \;\ge\; m\,\delta_{\min,i},
\]
where the last inequality uses $C_i(p) \ge 0$.
Since $p \circ r$ is feasible,
$m\,\delta_{\min,i} \le C_i(p \circ r) \le B_i$,
so $m \le B_i / \delta_{\min,i}$.
As $m$ is an integer, $m \le \lfloor B_i / \delta_{\min,i} \rfloor$.

Since this holds for every $i \in I_{\min}$,
\[
H(p) = m \;\le\; \min_{i \in I_{\min}} \left\lfloor \frac{B_i}{\delta_{\min,i}} \right\rfloor = \Lambda.
\]
Taking the supremum over all such $p$ gives $\bar{H}_1 \le \Lambda$.
\end{proof}

We now arrive at the central result of the paper.
The completion potential $H$, the residual budget monotonicity (Lemma~\ref{lem:residual-budget}), and the bounded skyline potential $\bar{H}_1$ provide all the ingredients needed to answer the convergence question posed at the beginning of this section.
\begin{definition}[Frontier potential floor]
\label{def:frontier-potential-floor}
The \emph{frontier potential floor} at time $t$ is
\[
h^*(t) \;:=\; \min_{p \in \mathcal{F}(t),\; H(p)<+\infty} H(p).
\]
\end{definition}

The following theorem shows that restricting extraction to the skyline does not merely preserve access to the closest-to-completion element---it induces a deterministic, monotone descent of $h^*$ toward zero, with at most $\bar{H}_1$ strict descent steps before a feasible solution is produced.

\begin{theorem}[Deterministic frontier-potential descent]
\label{thm:deterministic-layer-first}
Consider any policy that extracts exclusively from~$\mathcal{L}_1(t)$ — in particular Skyline-First — and suppose $\mathcal{S}(t) = \emptyset$.
Then:
\begin{enumerate}[label=(\roman*)]
\item At every time $t$, the skyline $\mathcal{L}_1(t)$ contains at least one element achieving $h^*(t)$.
      In particular, a skyline-only policy always has access to the frontier element closest to a feasible solution.
\item $h^*(0) \le \bar{H}_1$;
\item $h^*$ is monotonically non-increasing: $h^*(t+1) \le h^*(t)$;
\item whenever the extracted element $p$ satisfies $H(p) = h^*(t)$,
      strict progress occurs: \\ 
      $h^*(t+1) \le h^*(t) - 1$.
\end{enumerate}
In particular, at most $\bar{H}_1$ effective descent steps suffice to reach $h^* = 0$ and produce a feasible path.
\end{theorem}

\begin{proof}
\emph{Part~(i):}\enspace
Since a feasible path exists with $p \notin \mathcal{S}(t)$, Lemma~\ref{lem:frontier-persistence} ensures that $\mathcal{F}(t)$ contains an element with a feasible extension, so $\{p \in \mathcal{F}(t) : H(p) < +\infty\}$ is non-empty.
As $\mathcal{F}(t)$ is finite, the minimum $h^*(t)$ is well-defined and attained; let $p^*$ achieve $H(p^*) = h^*(t)$.

If $p^* \in \mathcal{L}_{1,\sig(p^*)}(t) \subseteq \mathcal{L}_1(t)$, we are done.
Otherwise, some $q \in \mathcal{L}_{1,\sig(p^*)}(t)$ dominates $p^*$, so Lemma~\ref{lem:residual-budget} gives $H(q) \le H(p^*) = h^*(t)$.

Since $h^*(t)$ is the global minimum, $H(q) = h^*(t)$, with $q \in \mathcal{L}_{1,\sig(p^*)}(t) \subseteq \mathcal{L}_1(t)$.

\emph{Part~(ii):}\enspace
The frontier is initialized with the trivial path $p_s = (s)$, so $\mathcal{F}(0) = \{p_s\}$.
By assumption, a feasible path $p$ from $s$ to some $t \in T$ exists; this path is a feasible extension of $p_s$, so $H(p_s) \le |p| < +\infty$ and $h^*(0) < +\infty$.
By~(i), some $p^* \in \mathcal{L}_1(0)$ achieves $H(p^*) = h^*(0)$.
Since $p^* \in \mathcal{L}_1(0)$, Definition~\ref{def:bounded-skyline-potential} gives $h^*(0) \le \bar{H}_1$.

\emph{Parts~(iii)--(iv):}\enspace
At step $t$, the policy extracts some $p \in \mathcal{L}_1(t)$ and inserts its successors.
If $H(p) > h^*(t)$, let $p^* \in \mathcal{F}(t)$ achieve $H(p^*) = h^*(t)$; since $p \neq p^*$, $p^*$ remains in $\mathcal{F}(t+1)$ and $h^*(t+1) \le H(p^*) = h^*(t)$, establishing~(iii).

Now suppose $H(p) = h^*(t)$.
We must have $h^*(t) \ge 1$: if $h^*(t) = 0$, then $H(p) = 0$, meaning $p$ is itself a feasible path discovered upon extraction, contradicting the premise.
By definition of $H$ (Definition~\ref{def:completion-potential}), $p$ admits at least one feasible extension.
Let $r = e_1 \circ \cdots \circ e_m$ be a shortest such extension, so that $m = H(p) = h^*(t) \ge 1$.
Let $p' = p \circ e_1$.
The remaining suffix $e_2 \circ \cdots \circ e_m$ is a feasible extension of $p'$ with $m-1$ edges, so $H(p') \le m - 1 = h^*(t) - 1$.
Moreover, $p'$ satisfies the budget: the cost function is non-decreasing ($\delta_i(\sigma, e, g) \ge g$), so
\[
C(p') \;\le\; C(p \circ r) \;\le\; B,
\]
where the last inequality holds because $p \circ r$ is feasible.
Hence $p'$ is not pruned and enters $\mathcal{F}(t+1)$, giving $h^*(t+1) \le H(p') \le h^*(t) - 1$, establishing~(iv).
\end{proof}

\subsection{Interpretation: Skyline Restriction Preserves Convergence}

Part~(i) is the key structural fact: at every time step, the frontier element closest to completion lies in the skyline, so a skyline-only policy never loses access to the most promising path.
Parts~(ii)--(iv) translate this accessibility into a deterministic monotone descent reaching a feasible solution in at most $\bar{H}_1$ effective steps.
By Lemma~\ref{lem:budget-potential-bound}, $\bar{H}_1 \le \Lambda = \min_{i \in I_{\min}} \lfloor B_i / \delta_{\min,i} \rfloor$, a quantity determined entirely by the budget and the minimum increments of the progressive dimensions.
Since these are parameters of the cost model, this bound is known before the search begins and can be used to anticipate the worst-case number of progress events.
Skyline restriction does not merely preserve convergence --- it structurally enforces it, with a number of descent steps bounded by cost-grid geometry alone.
The bound controls the number of strict descent steps; the total number of extractions between successive descents --- the plateau lengths --- remains uncontrolled (Section~\ref{sec:discussion}).

\paragraph{Running example (continued).}
Returning to the network of Figure~\ref{fig:running-graph}, consider the Skyline-First search starting from~$s$.
The initial frontier is $\mathcal{F}(0) = \{(s, \bot, (0,0,0))\}$, with $h^*(0) = 2$ since the shortest feasible extension from $s$ to $t$ uses two edges.
At step~1, the policy extracts $(s,\bot,(0,0,0))$ and expands it, producing frontier elements with signatures $(a,\mathsf{Z_1})$ at cost $(1,2,0)$ and $(b,\mathsf{Z_2})$ at cost $(1,1,0)$.
These belong to different per-signature frontiers and are both on the skyline.
At step~2, suppose we extract $(a,\mathsf{Z_1},(1,2,0))$; expanding it yields $(t,\mathsf{Z_1},(2,4,0))$ — a feasible path — so $h^*$ drops to~$0$ and the first solution is recorded.
This illustrates the two descent steps predicted by Theorem~\ref{thm:deterministic-layer-first}: from $h^*(0)=2$ to $h^*=1$ (after generating closer-to-target elements) and from $h^*=1$ to $h^*=0$ (solution found).

\subsection{Potential Stratification}

The completion potential induces a stratification of the frontier into levels
\[
\mathcal{F}_h(t) \;:=\; \bigl\{ p \in \mathcal{F}(t) : H(p) = h \bigr\}, \qquad h \in \mathbb{Z}_{\ge 0}.
\]
The minimal occupied level $\mathcal{F}_{h^*(t)}(t)$ represents the closest reachable region to feasibility.

A \emph{descent phase} is a maximal time interval during which the frontier potential floor decreases from some initial value down to~$0$, at which point a feasible path is discovered.
By Theorem~\ref{thm:deterministic-layer-first}, each strict decrease reduces $h^*$ by at least one.
Therefore, each phase contains at most $\bar{H}_1$ effective descent steps.

This suggests interpreting skyline-based search as a discrete dynamical system over a stratified state space, where feasible solutions correspond to absorbing states at level~$0$.

\begin{figure}[h!]
\centering
\begin{tikzpicture}[scale=1.0, every node/.style={font=\small}]

  \draw[->] (0,0) -- (0,5) node[above] {potential level $h$};

  \foreach \y/\label in {0/0,1/1,2/2,3/3,4/4} {
    \draw[gray!30] (0,\y) -- (9,\y);
    \node[left] at (0,\y) {$\mathcal{F}_{\label}$};
  }

  \foreach \x/\y in {
    1/4, 2.2/4,
    1.5/3, 2.8/3, 4/3,
    2/2, 3.5/2, 5/2,
    3/1, 4.5/1,
    4/0
  } {
    \filldraw[black] (\x,\y) circle (1.5pt);
  }

  \draw[very thick, red]
    (1,4) -- (1.5,3) -- (2,2) -- (3,1) -- (4,0);

  \draw[->, thick, red] (1,4) -- (1.5,3);
  \draw[->, thick, red] (1.5,3) -- (2,2);
  \draw[->, thick, red] (2,2) -- (3,1);
  \draw[->, thick, red] (3,1) -- (4,0);

  \node[right] at (1.1,4.2) {$h^*_0$};
  \node[right] at (1.6,3.2) {$h^*_1$};
  \node[right] at (2.1,2.2) {$h^*_2$};
  \node[right] at (3.1,1.2) {$h^*_3$};

  \filldraw[red] (4,0) circle (2.5pt);
  \node[below right] at (4,0) {solution};

  \node at (6.5,3.8) {\footnotesize skyline elements};
  \node[red] at (6.5,2.8) {\footnotesize descent trajectory};

\end{tikzpicture}
\caption{Stratification of the frontier by completion potential.
Each level $\mathcal{F}_h$ contains paths with identical residual completion distance.
The skyline (black points) spans multiple levels.
Red arrows mark the effective descent steps (Theorem~\ref{thm:deterministic-layer-first}): each occurs when a minimum-potential skyline element is extracted, but arbitrarily many non-descent extractions may separate consecutive steps.
Level $h=0$ represents feasible paths.}
\label{fig:stratified-descent}
\end{figure}
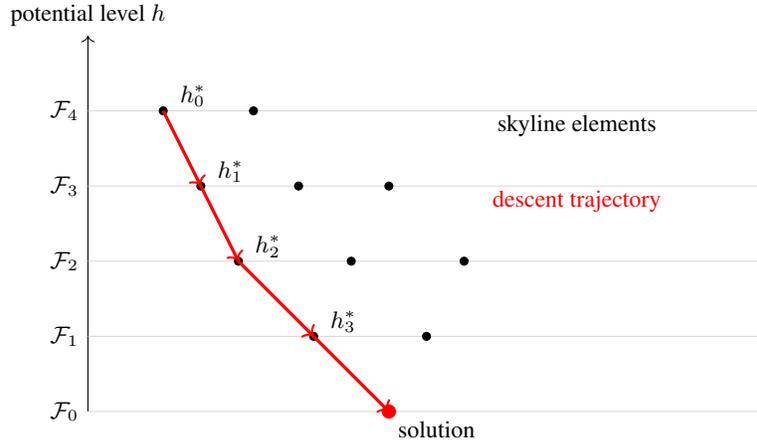
\FloatBarrier

Figure~\ref{fig:stratified-descent} provides a geometric interpretation of skyline-based search.
The search space is stratified by completion potential, and the frontier evolves within this layered structure.
The minimal potential $h^*(t)$ traces a descending path across levels, independently of the number of intermediate extractions performed within each layer.

\begin{figure}[h!]
\centering
\begin{tikzpicture}[scale=0.9, every node/.style={font=\small}]

\draw[->] (0,0) -- (10,0) node[right] {$t$};
\draw[->] (0,0) -- (0,4.5) node[above] {$h^*(t)$};

\foreach \y/\label in {0/0,1/1,2/2,3/3,4/4} {
  \draw[gray!30] (0,\y) -- (9.5,\y);
  \node[left] at (0,\y) {$\label$};
}

\draw[very thick]
  (0.5,4) -- (2.2,4)
  -- (2.2,3) -- (4.0,3)
  -- (4.0,2) -- (5.8,2)
  -- (5.8,1) -- (7.8,1)
  -- (7.8,0);

\node at (1.3,4.2) {\footnotesize plateau};
\node at (3.1,3.2) {\footnotesize plateau};
\node at (4.9,2.2) {\footnotesize plateau};
\node at (6.8,1.2) {\footnotesize plateau};

\filldraw (7.8,0) circle (2pt);
\node[below right] at (7.8,0) {solution};

\end{tikzpicture}
\caption{Temporal evolution of the frontier potential floor during a single descent phase.
The function $h^*(t)$ is monotone non-increasing.
Horizontal segments represent plateaus — extractions that do not reduce the minimal completion potential; vertical drops correspond to strict descent steps.
Reaching $h^*(t)=0$ yields a feasible path.}
\label{fig:descent-single}
\end{figure}
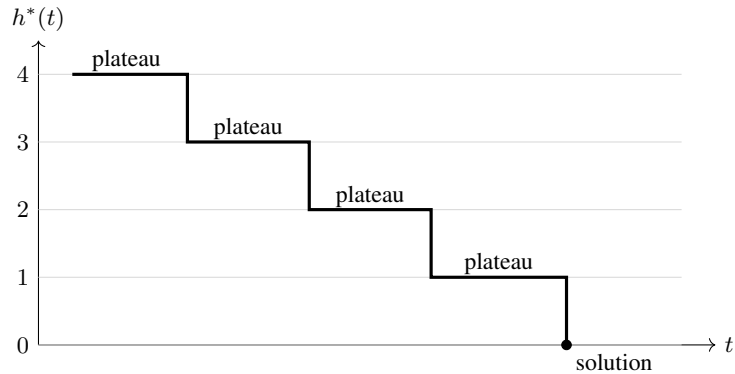
\FloatBarrier

Figure~\ref{fig:descent-single} shows the temporal dynamics of a single descent phase.
The search may spend an arbitrary number of extractions on a plateau, but each effective descent step strictly reduces the frontier potential floor.
This makes the number of true progress events structurally bounded, even though the temporal spacing between them is not.

\subsection{Amortized Descent for Multiple Solutions}

In many applications one seeks multiple diverse solutions.
We now extend the descent guarantee to $k$ successive solutions under a reachability assumption.
Each time a feasible path is discovered ($h^* = 0$), subsequent exploration must proceed through higher-potential levels until a new descent phase is initiated.
This induces a decomposition of the search into successive phases, each corresponding to the discovery of a new feasible solution.

\begin{proposition}[Amortized descent bound]
\label{cor:time-k}
Suppose the graph contains at least $k$ distinct feasible paths.
Their discovery requires at most
\[
k \cdot \bar{H}_1
\]
effective descent steps in total.
\end{proposition}

\begin{proof}
We proceed by induction on~$k$.
Define $\tau_j := \inf\{t \in \mathbb{N} : |\mathcal{S}(t)| \ge j\}$ for $j \ge 1$, with $\tau_0 := 0$.

\emph{Base case ($k=1$).}\enspace
Theorem~\ref{thm:deterministic-layer-first} gives at most $\bar{H}_1$ effective descent steps to discover the first feasible path.

\emph{Inductive step ($k \to k+1$).}\enspace
Assume the first $k$ solutions require at most $k \cdot \bar{H}_1$ descent steps and that the graph contains at least $k+1$ feasible paths.
At time $\tau_k$, at least one such path $p = e_1 \circ \cdots \circ e_m$ remains undiscovered.

Since $p \notin \mathcal{S}(\tau_k)$, Lemma~\ref{lem:frontier-persistence} ensures that $\mathcal{F}(\tau_k)$ contains an element with a feasible extension, so $\{p \in \mathcal{F}(\tau_k) : H(p) < +\infty\}$ is non-empty and $h^*(\tau_k) < +\infty$.

As $\mathcal{F}(\tau_k)$ is finite, $h^*(\tau_k)$ is attained by some $p^* \in \mathcal{F}(\tau_k)$.
By the same dominance argument as in Part~(i) of Theorem~\ref{thm:deterministic-layer-first}, some $q \in \mathcal{L}_1(\tau_k)$ satisfies $H(q) = h^*(\tau_k)$; since $q \in \mathcal{L}_1(\tau_k)$, Definition~\ref{def:bounded-skyline-potential} gives $h^*(\tau_k) \le \bar{H}_1$.
During the phase $[\tau_k, \tau_{k+1})$, the $(k+1)$-th solution has not yet been found, so the same monotonicity and strict-descent arguments as in the proof of Parts~(iii)--(iv) of Theorem~\ref{thm:deterministic-layer-first} apply: $h^*$ is non-increasing and each effective descent step reduces it by at least one.
Starting from $h^*(\tau_k) \le \bar{H}_1$, at most $\bar{H}_1$ such steps suffice to reach $h^* = 0$.
The total is at most $(k+1) \cdot \bar{H}_1$.
\end{proof}

\paragraph{Interpretation.}
This bound constrains the total number of effective progress events, not the total number of extractions: arbitrarily long plateaus may separate consecutive descent steps (see Section~\ref{sec:discussion}).
Conceptually, each phase reinitializes the same descent mechanism on a residual frontier: once a solution is found, a new completable prefix is guaranteed to exist, and the potential descent resumes from a fresh initial value bounded by $\bar{H}_1$.
The assumption requires only that $k$ feasible paths exist in the graph; the frontier invariant ensures that at least one completable prefix survives at each phase boundary.

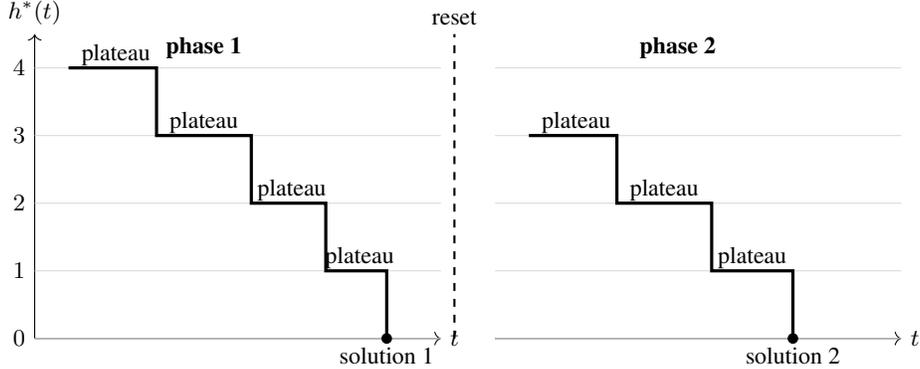
\begin{figure}[h!]
\centering
\begin{tikzpicture}[scale=0.9, every node/.style={font=\small}]

\draw[->] (0,0) -- (6,0) node[right] {$t$};
\draw[->] (0,0) -- (0,4.5) node[above] {$h^*(t)$};

\foreach \y/\label in {0/0,1/1,2/2,3/3,4/4} {
  \draw[gray!30] (0,\y) -- (6,\y);
  \node[left] at (0,\y) {$\label$};
}

\draw[very thick]
  (0.5,4) -- (1.8,4)
  -- (1.8,3) -- (3.2,3)
  -- (3.2,2) -- (4.3,2)
  -- (4.3,1) -- (5.2,1)
  -- (5.2,0);

\node at (2.5,4.3) {\textbf{phase 1}};
\node at (1.2,4.2) {\footnotesize plateau};
\node at (2.5,3.2) {\footnotesize plateau};
\node at (3.8,2.2) {\footnotesize plateau};
\node at (4.8,1.2) {\footnotesize plateau};

\filldraw (5.2,0) circle (2pt);
\node[below] at (5.2,0) {\footnotesize solution 1};

\draw[dashed, thick] (6.2,0) -- (6.2,4.5);
\node[above] at (6.2,4.5) {\footnotesize reset};

\draw[->] (6.8,0) -- (12.8,0) node[right] {$t$};

\foreach \y in {0,1,2,3,4} {
  \draw[gray!30] (6.8,\y) -- (12.8,\y);
}

\draw[very thick]
  (7.3,3) -- (8.6,3)
  -- (8.6,2) -- (10.0,2)
  -- (10.0,1) -- (11.2,1)
  -- (11.2,0);

\node at (9.5,4.3) {\textbf{phase 2}};
\node at (8.0,3.2) {\footnotesize plateau};
\node at (9.3,2.2) {\footnotesize plateau};
\node at (10.6,1.2) {\footnotesize plateau};

\filldraw (11.2,0) circle (2pt);
\node[below] at (11.2,0) {\footnotesize solution 2};

\end{tikzpicture}
\caption{Iterated descent phases across multiple solutions.
Each phase follows the same staircase pattern as Figure~\ref{fig:descent-single}.
The dashed separator indicates a phase reset: after a feasible path is discovered, the potential floor is reinitialized on the residual frontier, and a new descent begins.}
\label{fig:descent-phases}
\end{figure}
\FloatBarrier

Figure~\ref{fig:descent-phases} illustrates the iterated descent mechanism.
Each solution discovery triggers a phase reset, after which the potential floor is reinitialized at $h^*(\tau_j) \le \bar{H}_1$ and a new descent begins.

\section{Computational Interpretation}
\label{sec:complexity}

At any time $t$, the algorithm maintains the frontier $\mathcal{F}(t)$, the skyline $\mathcal{L}_1(t)$, and the solution set $\mathcal{S}(t)$, subject to three invariants: every extracted element belongs to $\mathcal{L}_1(t)$, the skyline width satisfies $|\mathcal{L}_1(t)| \le W(t) = |\Sigma_\mathcal{F}(t)| \cdot B^*$, and the algorithm halts once the stopping certificate of Section~\ref{sec:certified-stopping} holds.
All control decisions derive from the dominance structure alone.
We now collect these structural parameters into a runtime expression that reveals how computational cost is governed by cost-grid geometry rather than path combinatorics.

\begin{proposition}[Structural runtime bound]
\label{prop:runtime}
Let $t^*$ be the first time at which the stopping certificate holds.
The total number of cost-update and dominance-comparison operations performed by any skyline-only policy is at most
\[
O\!\left(
t^*
\cdot \Delta \cdot W \cdot d
\right),
\]
where $\Delta$ is the maximum out-degree, $W = \max_t |\Sigma_\mathcal{F}(t)| \cdot B^*$ is the peak skyline width, and $d$ is the number of cost dimensions.
\end{proposition}
\begin{proof}
The algorithm performs at most $t^*$ extraction steps before the stopping certificate holds.
At each step, the extracted element is expanded along at most $\Delta$ outgoing edges, producing at most $\Delta$ successor paths.
For each successor, two operations are required:
(i)~a cost update, which evaluates $\delta_i(\sigma, e, C_i(p))$ for each of the $d$ dimensions, costing $O(d)$;
(ii)~a dominance check against the current skyline $\mathcal{L}_1(t)$ to determine whether the successor is dominated.
By Proposition~\ref{thm:layer-width}, $|\mathcal{L}_1(t)| \le W$ at every time $t$, so each dominance check compares $d$ components against at most $W$ skyline elements, costing $O(W \cdot d)$.
The per-extraction cost is therefore $O(\Delta \cdot (d + W \cdot d)) = O(\Delta \cdot W \cdot d)$, since $W \ge 1$.
Summing over all $t^*$ extractions yields the stated bound.
\end{proof}

Unlike classical multi-objective label-setting, where runtime depends on the total number of generated labels, the bound above is governed by cost-grid cardinality ($B^*$), skyline width ($W$), and the deterministic stopping certificate ($t^*$) — structural parameters of the problem geometry.
This does not eliminate combinatorial explosion, but shifts the controlling factor from path enumeration to cost-space geometry.
The factor $B^*$ is exponential in $d$, which is inherent to dominance-based schemes; the framework targets the fixed-dimension regime ($d$ typically $2$--$5$) where $B^*$ is a moderate constant.
The skyline width $W = \max_t |\Sigma_\mathcal{F}(t)| \cdot B^*$ also depends on the number of active signatures $|\Sigma_\mathcal{F}(t)|$, which is bounded by $|\Sigma|$ but may be much smaller in practice if only a fraction of the signature space is reachable from~$s$.

\paragraph{Certificate verification cost.}
The stopping condition requires verifying the resolution certificate over the current skyline.
In the present framework, this reduces to componentwise comparisons between each skyline element and the set of discovered solutions, using the constant vector $\delta_{\min}$.
Certificate verification therefore incurs a cost of $O(|\mathcal{L}_1(t)| \cdot |\mathcal{S}(t)| \cdot d)$ per step, bounded by $O(W \cdot |\mathcal{S}(t)| \cdot d)$.
No additional heuristic or precomputation is required: the certificate relies solely on the structural properties of the cost model.
In practice, the certificate need not be verified at every extraction step; periodic verification suffices without affecting the theoretical guarantees.
In more general settings, tighter lower bounds could be incorporated to accelerate certification, at the expense of additional computational overhead.

\section{Discussion}
\label{sec:discussion}

\subsection{On the Origin of Structural Assumptions}

The assumptions underlying our analysis reflect structural properties of real-world infrastructure graphs rather than restrictive modeling choices.
In telecom and datacenter digital twins, path construction is not performed over arbitrary graphs: transitions are constrained by local engineering rules, costs are discretized by design, and feasible solutions follow a directional progression imposed by network topology.
These properties induce a regular, bounded cost space that departs significantly from worst-case combinatorial settings.
Our results show that, in this regime, Pareto dominance is not merely a filtering criterion but a sufficient mechanism to control the entire traversal process.
This suggests that the perceived hardness of multi-criteria path search is, in part, an artifact of overly general formulations that do not exploit the structural constraints present in real systems.

\newpage
\subsection{On Plateau Lengths and Time-to-First-Solution}

Theorem~\ref{thm:deterministic-layer-first} establishes a deterministic descent property for the frontier potential, but does not control the number of extractions required between successive descent steps.

In particular, the duration of plateaus — intervals during which $h^*(t)$ remains constant — depends on the interaction between skyline selection and frontier evolution.
Newly generated successors may enter the skyline without decreasing the completion potential, preventing a simple ranking argument.
This turnover in the skyline composition---where the first layer changes without reducing the minimum potential---makes the number of extractions at constant potential level resistant to local counting arguments.

A trivial upper bound on the total number of extractions can be obtained from the finiteness of the quantized cost space.
Since the number of reachable signatures is finite and each signature admits at most $B^*$ cost vectors, any skyline-only exploration terminates after at most $|\Sigma| \cdot B^*$ extractions.

However, deriving tighter bounds on plateau lengths remains an open problem.
In particular, it is unclear under which structural conditions on the graph or the cost space one can guarantee tighter bounds on plateau lengths.

While plateau lengths are not explicitly bounded, termination is guaranteed after a finite number of extractions under the structural assumptions.
The absence of tighter bounds is the main limitation of the current framework: the theory controls progress events but not the idle time between them.
Closing this gap---by identifying structural conditions that bound plateau lengths---is the key step toward a fully predictive complexity theory.

Beyond plateau length analysis, several directions emerge for extending the present framework.

\subsection{Extensions and Generalizations}

A first direction concerns the robustness of the skyline-first control mechanism under relaxed structural conditions.
In particular, it remains open whether approximate variants of Pareto dominance---such as $\varepsilon$-dominance or layer compression---preserve the descent and termination guarantees, and how such relaxations affect layer width and stopping certificates.

A second direction concerns the interaction with heuristic guidance.
While the current results show that no external mechanism is required to ensure progress and termination, it is unclear whether heuristics can accelerate plateau resolution without breaking the structural guarantees induced by Pareto layers.

Finally, extending the framework beyond the current regime---toward non-Markovian costs, dynamic graphs, or partially continuous cost spaces---would help delineate the boundary between structural and combinatorial complexity in multi-criteria traversal.

\section{Conclusion}
\label{sec:conclusion}

Multi-criteria graph search has traditionally been approached through external control mechanisms---heuristics, scalarization, or population-based exploration---while Pareto dominance has been confined to a passive role.
This work shows that, under appropriate structural conditions on the cost space, Pareto geometry alone is sufficient to control both the progression and termination of the search.

By organizing the frontier into Pareto layers and restricting extraction to the skyline, we obtain a traversal process driven by deterministic descent in a completion potential, together with a stopping criterion based on dominance coverage.
This yields a unified framework in which scheduling and termination emerge directly from the geometry of the cost space, without requiring auxiliary guidance.

These results suggest a shift in perspective: the difficulty of multi-criteria search is not solely intrinsic to the problem, but also depends on whether the underlying cost structure is exploited.
In structured regimes, Pareto dominance ceases to be a passive filter and becomes a control principle.

This opens the way to structure-aware traversal engines for large-scale infrastructure systems, and more broadly, to a structural theory of search in multi-criteria cost spaces.


\bibliographystyle{unsrt}
\bibliography{references}

\end{document}